\documentclass{article} 
\usepackage{mefomo_2023,times}

\usepackage{amsmath,amsfonts,amssymb,amsthm}

\usepackage{microtype}
\usepackage{graphicx}
\usepackage{subfig}

\usepackage[symbol]{footmisc}
\renewcommand{\thefootnote}{\fnsymbol{footnote}}

\iclrfinalcopy

\usepackage{amsmath,amsfonts,bm}

\def\eqref#1{equation~\ref{#1}}

\def\1{\bm{1}}

\DeclareMathAlphabet{\mathsfit}{\encodingdefault}{\sfdefault}{m}{sl}
\SetMathAlphabet{\mathsfit}{bold}{\encodingdefault}{\sfdefault}{bx}{n}

\title{A Tale of Two Circuits: Grokking as Competition of Sparse and Dense Subnetworks}

\author{William Merrill\thanks{Equal contribution} \hspace{0.1mm}, Nikolaos Tsilivis\footnotemark[1] \hspace{0.1mm} \& Aman Shukla \\
New York University
}

\usepackage[hidelinks]{hyperref}
\usepackage{url}
\usepackage{cleveref}

\usepackage{xcolor}

\usepackage{amsthm}
\newtheorem{proposition}{Proposition}
\Crefformat{proposition}{Proposition #2#1#3}

\usepackage{float}
\usepackage{paralist}

\begin{document}

\maketitle

\begin{abstract}
Grokking is a phenomenon where a model trained on an algorithmic task first overfits but, then, after a large amount of additional training, undergoes a phase transition to generalize perfectly. We empirically study the internal structure of networks undergoing grokking on the sparse parity task, and find that the grokking phase transition corresponds to the emergence of a sparse subnetwork that dominates model predictions. On an optimization level, we find that this subnetwork arises when a small subset of neurons undergoes rapid norm growth, whereas the other neurons in the network decay slowly in norm. Thus, we suggest that the grokking phase transition can be understood to emerge from \emph{competition} of two largely distinct subnetworks: a dense one that dominates before the transition and generalizes poorly, and a sparse one that dominates afterwards.

\end{abstract}

\section{Introduction}
\renewcommand*{\thefootnote}{\arabic{footnote}}


Grokking \citep{powers2022grokking,barak2022hidden} is a curious generalization trend for neural networks trained on certain algorithmic tasks. Under grokking, the network's accuracy (and loss) plot displays two phases. Early in training, the training accuracy goes to $100\%$, while the generalization accuracy remains near chance. Since the network appears to be simply memorizing the data in this phase, we refer to this as the \emph{memorization phase}. Significantly later in training, the generalization accuracy spikes suddenly to $100\%$, which we call the \emph{grokking transition}.

This mysterious pattern defies conventional machine learning wisdom: after initially overfitting, the model is somehow learning the correct, generalizing behavior without any disambiguating evidence from the data. Accounting for this strange behavior motivates developing a theory of grokking rooted in optimization.
Moreover, grokking resembles so-called emergent behavior in large language models \citep{barret2022emergent}, where performance on some (often algorithmic) capability remains at chance below a critical scale threshold, but, with enough scale, shows roughly monotonic improvement.
We thus might view grokking as a controlled test bed for emergence in large language models, and hope that understanding the dynamics of grokking could lead to hypotheses for analyzing such emergent capabilities. Ideally, an effective theory for such phenomena should be able to understand the causal mechanisms behind the phase transitions, predict on which downstream tasks they could happen, and disentangle the statistical (number of data) from the computational (compute time, size of network) aspects of the problem.

While grokking was originally identified on algorithmic tasks, \citet{liu2023omnigrok} show it can be induced on natural tasks from other domains with the right hyperparameters. Additionally, grokking-like phase transitions have long been studied in the statistical physics community \citep{EnVa01}, albeit in a slightly different setting (online gradient descent, large limits of model parameters and amount of data etc.).
Past work analyzing grokking has reverse-engineered the network behavior in Fourier space \citep{nanda2023progress} and found measures of progress towards generalization before the grokking transition \citep{barak2022hidden}.
\citet{thilak2022the} observe a ``slingshot'' pattern during grokking: the final layer weight norm follows a roughly sigmoidal growth trend around the grokking phase transition.
This suggests grokking is related to the magnitude of neurons within the network, though without a clear theoretical explanation or account of individual neuron behavior.

In this work, we aim to better understand grokking on sparse parity \citep{barak2022hidden} by studying the \textit{sparsity} and \textit{computational structure} of the model over time. We empirically demonstrate a connection between grokking, emergent sparsity, and competition between different structures inside the model (\Cref{fig:three_phases}).
We first show that, after grokking, network behavior is controlled by a sparse subnetwork (but by a dense one before the transition).
Aiming to better understand this sparse subnetwork, we then demonstrate that the grokking phase transition corresponds to accerelated norm growth in a \emph{specific} set of neurons, and norm decay elsewhere.
After this norm growth, we find that the targeted neurons quickly begin to dominate network predictions, leading to the emergence of the sparse subnetwork.
We also find that the size of the sparse subnetwork corresponds to the size of a disjunctive normal form circuit for computing parity, suggesting this may be what the model is doing.
Taken together, our results suggest grokking arises from targeted norm growth of specific neurons within the network. This targeted norm growth sparsifies the network, potentially enabling generalizing discrete behavior that is useful for algorithmic tasks.

\begin{figure}
    \centering
    \includegraphics[scale=0.5]{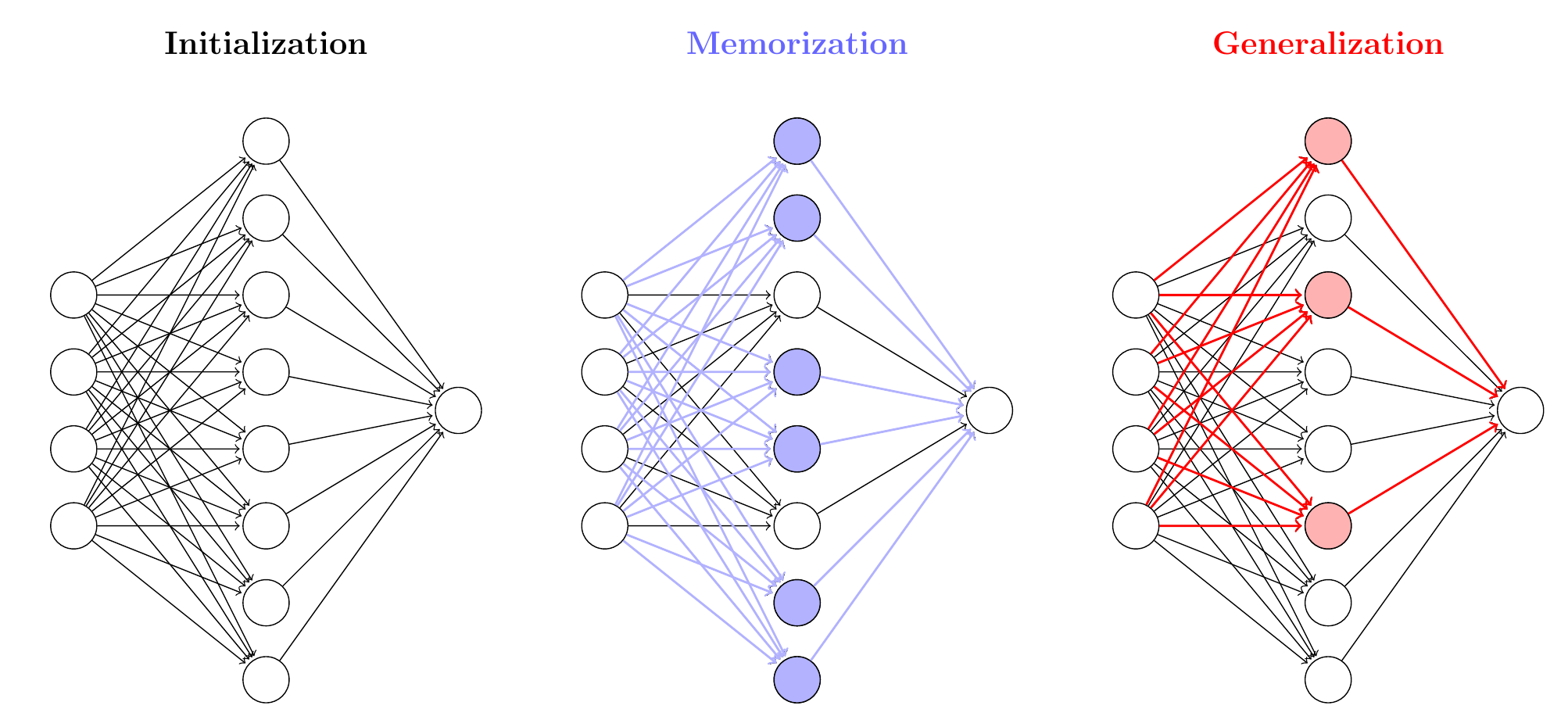}
    \caption{An illustration of the structure of a neural network during training in algorithmic tasks. Neural networks often exhibit a memorization phase that corresponds to a dense network, followed by the generalization phase where a sparse, largely disjoint to the prior one, model takes over.}
    \label{fig:three_phases}
\end{figure}

\section{Tasks, Models, and Methods} \label{sec:methods}

\paragraph{Sparse Parity Function.} We focus on analyzing grokking in the problem of learning a sparse $(n, k)$-parity function \citep{barak2022hidden}.
A $(n, k)$-parity function takes as input a string $x \in \{ \pm 1 \}^n$ returns $\pi(x) = \prod_{i \in S} x_i \in \{ \pm 1 \}$, where $S$ is a fixed, \textit{hidden} set of $k$ indices. The training set consists of $N$ i.i.d. samples of $x$ and $\pi(x)$.
We call the $(n, k)$-parity problem \emph{sparse} when $k \ll n$, which is satisfied by our choice of $n = 40$, $k=3$, and $N=1000$. 

\paragraph{Network Architecture.}
Following \citet{barak2022hidden}, we use a $1$-layer ReLU net:
\begin{align*}\label{eq:arch}
    f(x) &= u^\intercal \sigma (Wx + b) \\
    \Tilde{y} &= \mathrm{sgn}\left (f(x)\right) ,
\end{align*}
where $\sigma (x) = \max(0, x)$ is ReLU, $u \in \mathbb{R}^p, W \in \mathbb{R}^{p \times n}$, and $ b \in \mathbb{R}^p$.  We minimize the hinge loss $\ell(x, y) = \max(0, 1-f(x) y)$, using stochastic gradient descent (batch size $B$) with constant learning rate $\eta$ and (potential) weight decay of strength $\lambda$ (that is, we minimize the regularized loss $\ell(x, y) + \lambda \| \theta \|_2$, where $\theta$ denotes all the parameters of the model). 
Unless stated otherwise, we use weight decay $\lambda = 0.01$, learning rate $\eta = 0.1$ batch size $B = 32$, and hidden size $p = 1000$.
We train each network 5 times, varying the random seed for generating the train set and training the model, but keeping the test set of $100$ points fixed.

\subsection{Active Subnetworks and Effective Sparsity} \label{sec:sparsity}
We use a variant of weight magnitude pruning \citep{MoSm89} to find active subnetworks that control the full network predictions. The method assumes a given support of input data $\mathcal X$. Let $f$ be the network prediction function and $f_{k}$ be the prediction where the $p-k$ neurons with the least-magnitude incoming edges have been pruned. We define the \emph{active subnetwork} of $f$ as $f_k$ where $k$ is minimal such that, for all $x \in \mathcal X$, $f(x) = f_{k}(x)$.

We will use the active subnetwork to identify important structures within a network during grokking. We can also naturally use it to measure sparsity: we define the \emph{effective sparsity} of $f$ as the number of neurons in the hidden layer of the active subnetwork of $f$.




\section{Results}


We see in \Cref{fig:accloss_spars} that our sparse parity task indeed displays grokking, both in accuracy and loss. We now turn to analyzing the internal network structure before, during, and after grokking. We refer to Appendix \ref{app:other_configs} for additional configurations (smaller weight decay or larger parity size) that support our findings\footnote{Code available on \url{https://github.com/Tsili42/parity-nn}}.

\begin{figure}
    \centering
    \includegraphics[scale=0.27]{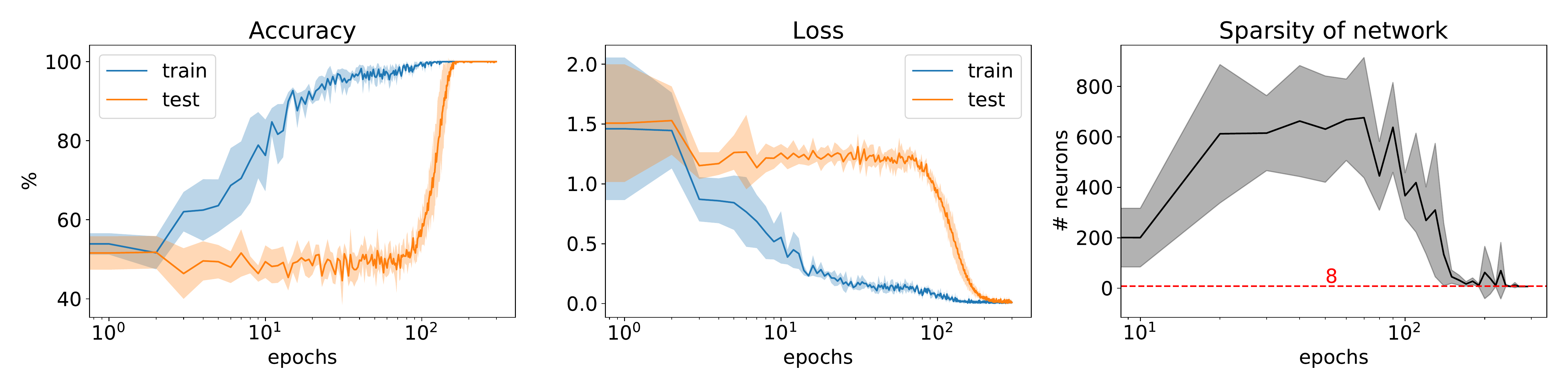}
    \caption{Accuracy (left), Average Loss (middle) and Effective Sparsity (right) during training of an FC network on $(40, 3)$ parity. Generalization accuracy suddenly jumps from random chance to flawless prediction concurrent with sparsification of the model. Shaded areas show randomness over the training dataset sampling, model initialization, and stochasticity of SGD (5 random seeds).}
    \label{fig:accloss_spars}
\end{figure}

\subsection{Grokking Corresponds to Sparsification}

\Cref{fig:accloss_spars} (right) shows the effective sparsity~(number of active neurons; cf. \Cref{sec:sparsity}) of the network over time. Noticeably, it becomes orders of magnitude sparser as it starts generalizing to the test set, and crucially, this phase transition happens at the same time as the loss phase transition. This can be directly attributed to the norm regularization being applied in the loss function, as it kicks in right after we reach (almost) zero in the data-fidelity part of the loss. Interestingly, this phase transition can be calculated solely from the training data but correlates with the phase transition in the test accuracy.

\citet{nanda2023progress} observe sparsity in the Fourier domain after grokking, whereas we have found it in the conventional network structure as well.
Motivated by the discovery of this sparse subnetwork, we now turn our attention to understanding why this subnetwork emerges and its structure.


\subsection{Selective Norm Growth Induces Sparsity During Grokking}

Having identified a sparse subnetwork of neurons that emerges to control network behavior after the grokking transition, we study the properties of these neurons throughout training (before the formation of the sparse subnetwork). \Cref{fig:subnetworks} (left) plots the average neuron norm for three sets of neurons: the neurons that end up in the sparse subnetwork, the complement of those neurons, and a set of random neurons with the same size as the sparse subnetwork.
We find that the 3 networks have similar average norm up to a point slightly before the grokking phase transition, at which the generalizing subnetwork norm begins to grow rapidly.

In \Cref{fig:subnetworks} (right), we measure the faithfulness of the neurons in the sparse subnetwork over time: in other words, the ability of these neurons alone to reconstruct the full network predictions on the test set, measured as accuracy. The grokking phase transition corresponds to these networks emerging to fully explain network predictions, and we believe this is likely a causal effect of norm growth.\footnote{Conventional machine learning wisdom associates small weight norm with sparsity, so it may appear counterintuitive that \emph{growing} norm induces sparsity. We note that the growth of selective weights can lead to effective sparsity because the large weights dominate linear layers \citep{merrill-etal-2021-effects}.}
The fact that the performance of its complement degrades after grokking supports the conclusion that the sparse network is ``competing'' with another network to inform model predictions, and that the grokking transition corresponds to a sudden switch where the sparse network dominates model output.

The element of competition between the different circuits is further evident when plotting the norm of individual neurons over time. \Cref{fig:ind_neurons} in the Appendix shows that neurons active during the memorization phase slightly grow in norm before grokking but then ``die out", while the the neurons of sparse subnetwork are inactive during memorization and then explode in norm. The fact that the model is overparameterized allows this kind of competition to take place.

\subsection{Subnetwork Computational Structure}

\paragraph{Sparse Subnetwork.} Across $5$ runs, the sparse subnetwork has size $\{6, 6, 6, 8, 8\}$. This suggests that the network may be computing the parity via a representation resembling disjunctive normal form (DNF), via the following argument.
A standard DNF construction uses $2^k =8$ neurons to compute the parity of $k=3$ bits (\Cref{thm:dnf-general}). We also derive a modified DNF net that uses only $6$ neurons to compute the parity of $3$ bits (\Cref{thm:dnf-special}). Since our sparse subnetwork always contains either $6$ or $8$ neuron, we speculate it may always be implementing a variant of these constructions. However, there is an even smaller network computing an $(n, 3)$-parity with only 4 neurons via a threshold-gate construction, but it does not appear to be found by our networks (\Cref{thm:threshold}).

\paragraph{Dense Subnetwork.} The network active during the so-called memorization phase is not exactly memorizing. Evidence for this claim comes from observing grokking on the binary operator task of \citet{powers2022grokking}. For the originally reported division operator task, the network obtains near zero generalization prior to grokking (\Cref{fig:memorization_evidence}, right). However, switching the operator to addition, the generalization accuracy is above chance before grokking (\Cref{fig:memorization_evidence}, left). We hypothesize this is because the network, even pre-grokking, can generalize to unseen data since addition (unlike division) is commutative. In this sense, it is not strictly memorizing the training data.

\begin{figure}
    \centering
    \includegraphics[scale=0.3]{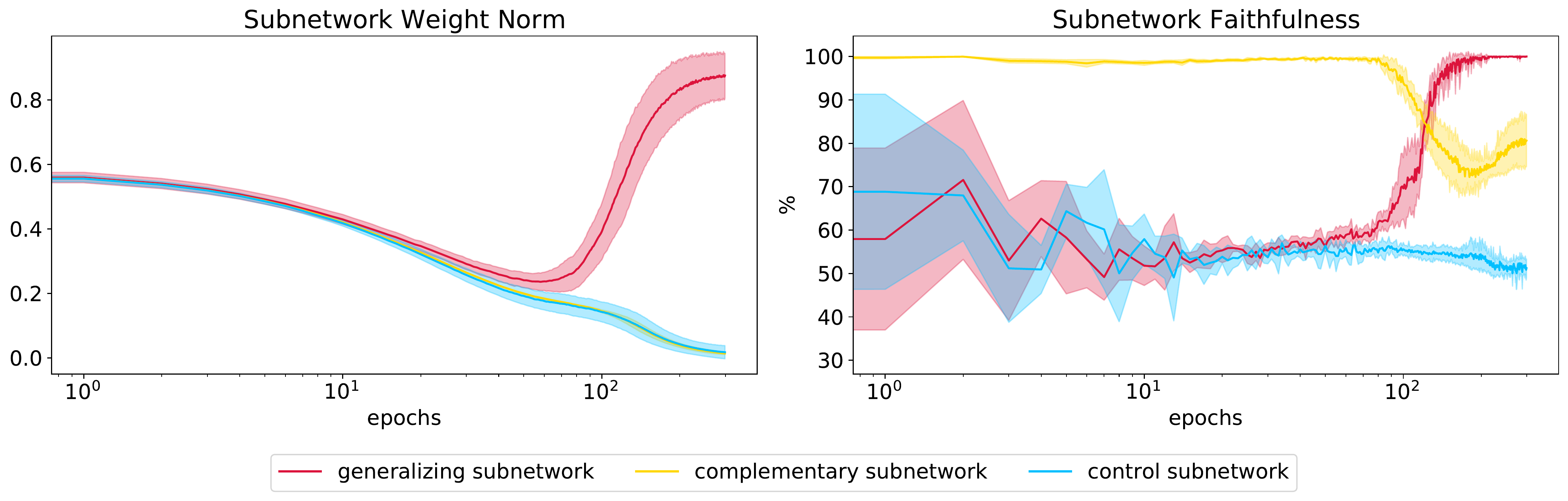}
    \caption{Left: Average norm of different subnetworks during training. Right: Agreement between the predictions of a subnetwork and the full network on the test set. The \textit{generalizing} subnetwork is the final sparse net, the \textit{complementary} subnetwork is its complement, and the \textit{control} subnetwork is a random network with the same size as the generalizing one.}
    \label{fig:subnetworks}
\end{figure}

\section{Conclusion}

We have shown empirically that the grokking phase transition, at least in a specific setting, arises from competition between a sparse subnetwork and the (dense) rest of the network.
Moreover, grokking seems to arise from selective norm growth of this subnetwork's neurons.
As a result, the sparse subnetwork is largely inactive during the memorization phase, but soon after grokking, fully controls model prediction.

We speculate that norm growth and sparsity may facilitate emergent behavior in large language models similar to their role in grokking. 
As preliminary evidence, \citet{merrill-etal-2021-effects} observed monotonic norm growth of the parameters in T5, leading to ``saturation'' of the network in function space.\footnote{Saturation measures the discreteness of the network function, but may relate to effective sparsity.} More promisingly, \citet{dettmers2022gptint} observe that a targeted subset of weights in pretrained language models have high magnitude, and that these weights overwhelmingly explain model predictions.
It would also be interesting to extend our analysis of grokking to large language models: specifically, does targeted norm growth subnetworks of large language models \citep{dettmers2022gptint} facilitate emergent behavior?

\section{Acknowledgements}

This material is based upon work supported by the National Science Foundation under NSF Award 1922658.

\bibliography{references.bib}
\bibliographystyle{iclr2023_conference}

\appendix
\section{Binary Operator Experiments}

We trained a decoder only transformer with 2 layers, width 128, and 4 attention heads \citep{powers2022grokking}. In both operator settings, we used the AdamW optimizer, with a learning rate of $10^{-3}$, $\beta_1 = 0.9$ and $\beta_2 = 0.98$, weight decay equal to 1, batch size equal to 512, 9400 sample points and an optimization limit of $10^5$ updates. We repeated the experiments for both operators with 3 random seeds and aggregated the results.  

\begin{figure}[H]
    \centering
    \includegraphics[scale=0.4]{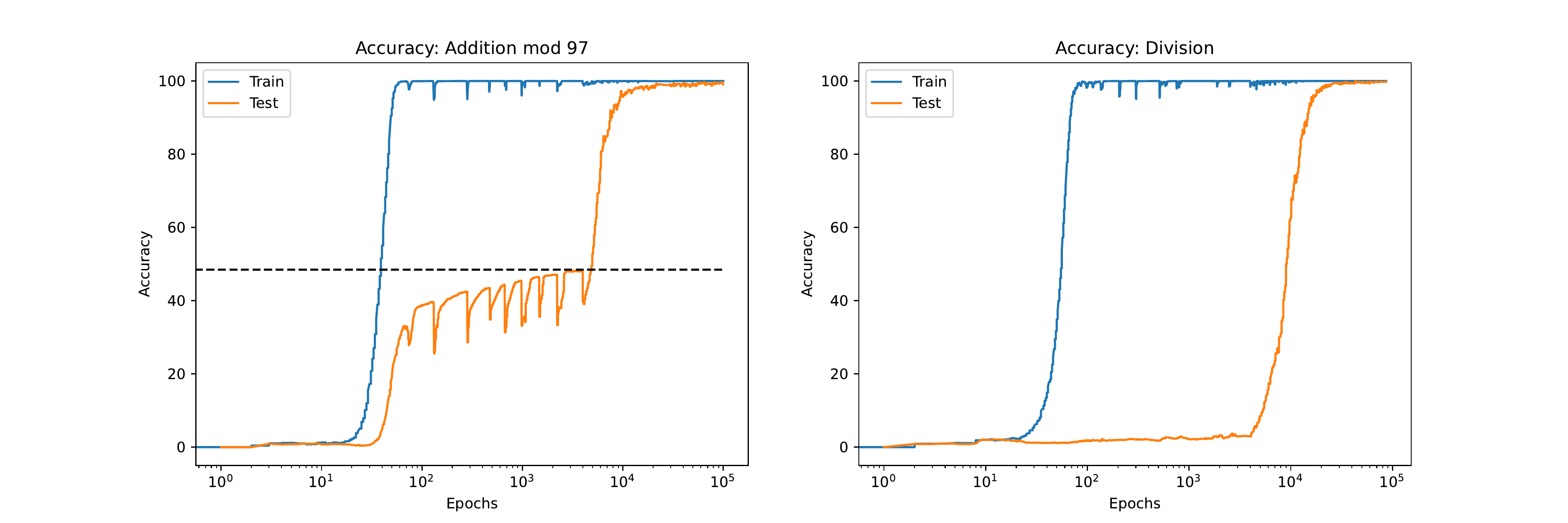}
    \caption{Accuracy curves for addition (left) and division (right). For the addition operator, the dashed line represents the \% of dataset that can be solved by commuting test points and then looking them up in the memorized training set.
    The generalization accuracy before grokking matches this level, suggested that the network has learned to generalize the commutative property of addition before it learns to generalize fully.}
    \label{fig:memorization_evidence}
\end{figure}

\section{Additional plots}

\begin{figure}[H]
    \centering
    \includegraphics[scale=0.42]{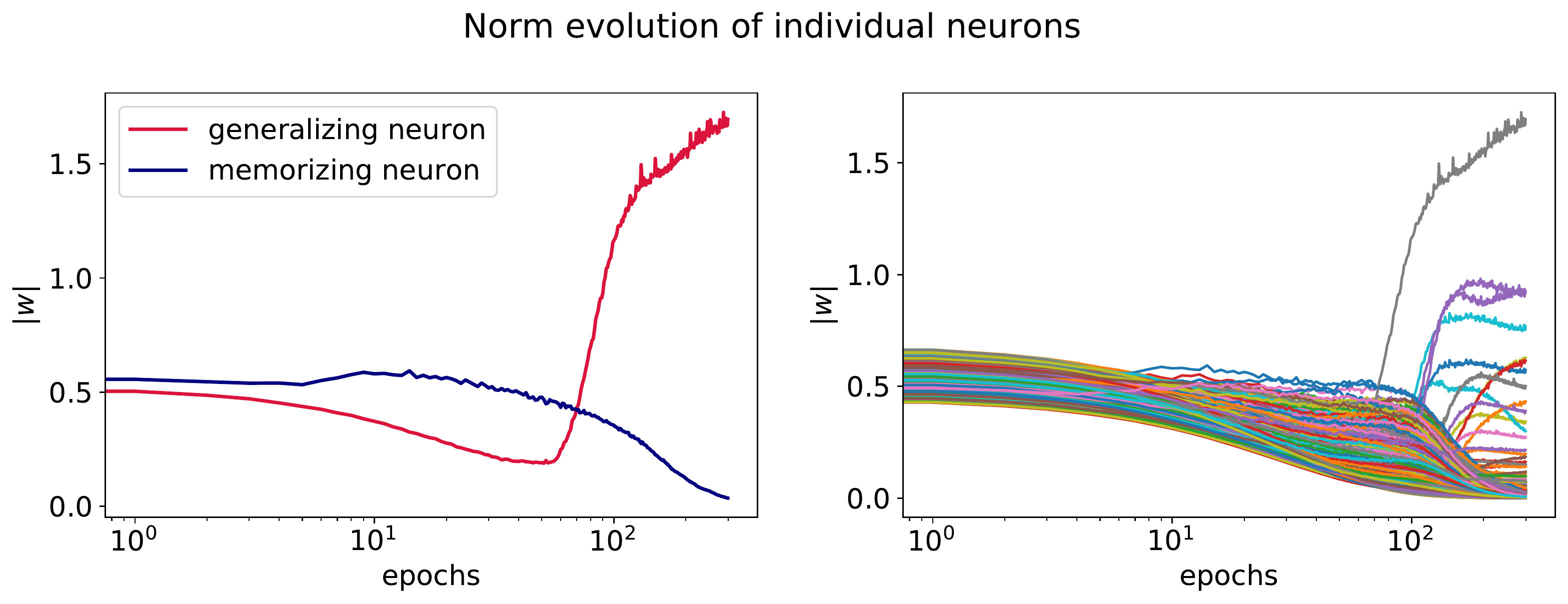}
    \caption{Weight norm of individual neurons during training. Left: Evolution of the dominant neurons during the memorization epoch (first time we hit $>$ 98\% train accuracy) and final epoch (that corresponds to the generalizing subnetwork). Right: Weight norm over time for all neurons. Notice that most of them are driven to 0.}
    \label{fig:ind_neurons}
\end{figure}

\section{Additional Configurations}\label{app:other_configs}

We provide accuracy, loss, sparsity, subnetwork norm and subnetwork faithfulness plots for smaller weight decay (Figures \ref{fig:wd0001_acc} and \ref{fig:wd0001_sub}), and for larger parity size (Figures \ref{fig:k4_acc} and \ref{fig:k4_sub}). The experimental observations are consistent with those of the main body of the paper.

\begin{figure}
    \centering \includegraphics[scale=0.27]{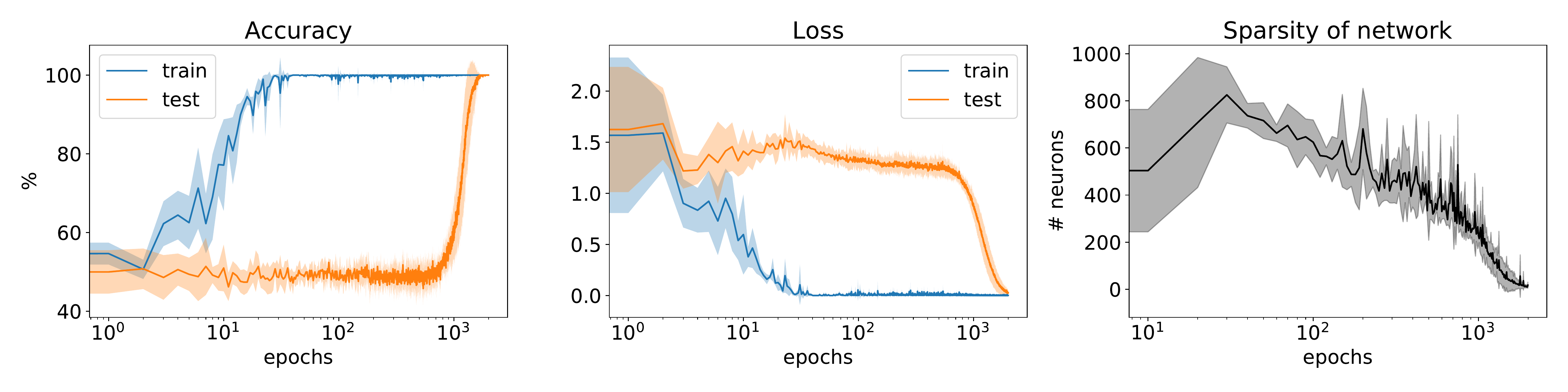}
    \caption{Reproduction of Figure \ref{fig:accloss_spars} for smaller weight decay $\lambda=0.001$ (the rest of the hyperparameters are the same as in the standard setup). Accuracy (left), Average Loss (middle) and Effective Sparsity (right) during training of an FC network on $(40, 3)$-parity.}
    \label{fig:wd0001_acc}
\end{figure}

\begin{figure}
    \centering
    \includegraphics[scale=0.3]{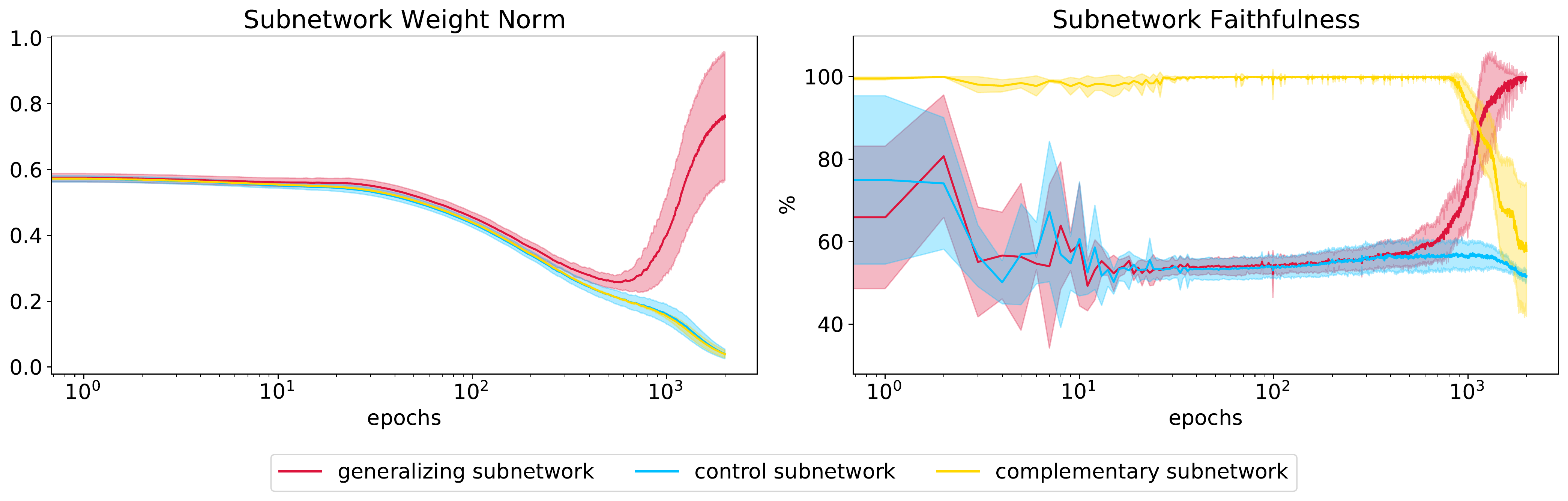}
    \caption{Reproduction of Figure \ref{fig:subnetworks} for smaller weight decay $\lambda=0.001$ (the rest of the hyperparameters are the same as in the standard setup). Left: Average norm of different subnetworks during training. Right: Agreement between the predictions of a subnetwork and the full network on the test set.}
    \label{fig:wd0001_sub}
\end{figure}

\begin{figure}
    \centering
    \includegraphics[scale=0.27]{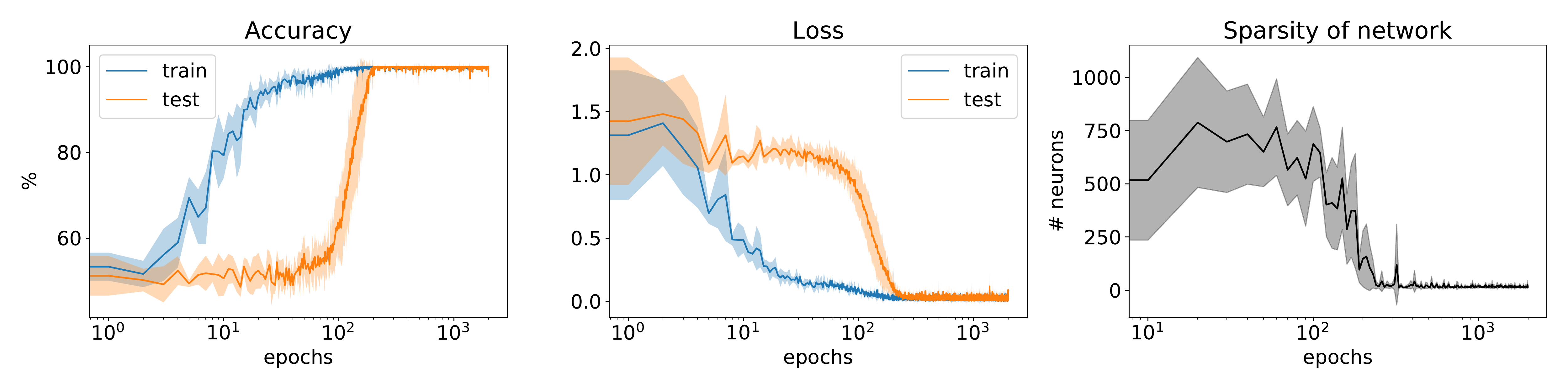}
    \caption{Reproduction of Figure \ref{fig:accloss_spars} for larger parity size $k=4$ (the rest of the hyperparameters are the same as in the standard setup). Accuracy (left), Average Loss (middle) and Effective Sparsity (right) during training of an FC network on $(40, 4)$ parity.}
    \label{fig:k4_acc}
\end{figure}

\begin{figure}
    \centering
    \includegraphics[scale=0.3]{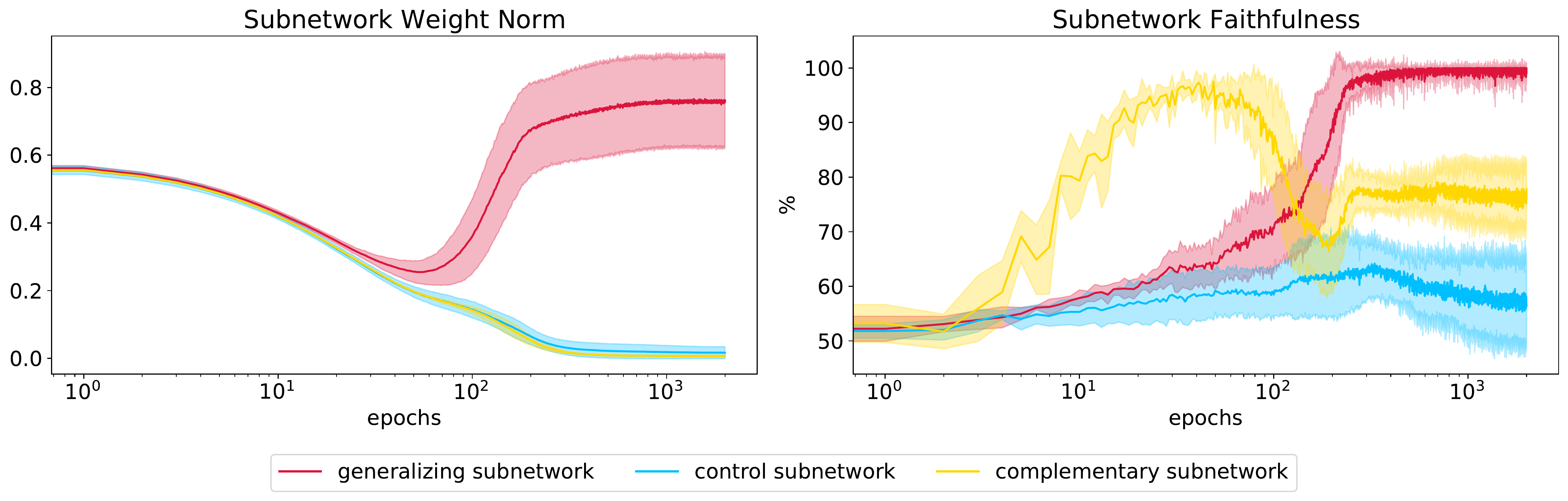}
    \caption{Reproduction of Figure \ref{fig:subnetworks} for larger parity size $k=4$ (the rest of the hyperparameters are the same as in the standard setup). Left: Average norm of different subnetworks during training. Right: Agreement between the predictions of a subnetwork and the full network on the test set.}
    \label{fig:k4_sub}
\end{figure}

\section{Computing Parity with Neural Nets} \label{sec:parity}

We say that a neural net of the form defined in \autoref{sec:methods} computes parity iff its output is positive when the parity is $1$ and negative otherwise.

We first show a general way to represent parity in ReLU networks for any parity size $k$. This construction requires $2^k$ hidden neurons.

\begin{proposition} \label{thm:dnf-general}
For any $n$, there exists a 1-layer ReLU net with $2^k$ neurons that computes $(n, k)$-parity.
\end{proposition}

\begin{proof}
We use each $2^k$ neurons to match a specific configuration of the $k$ parity bits by using the first affine transformation to implement an AND gate (note that the bias term is crucial here). In the output layer, we add positive weight on edges from neurons corresponding to configurations with parity $1$ and negative weight for neurons corresponding to configurations with parity $-1$.
\end{proof}

In the case where $k=3$, we show that there is a simpler construction with $6$ neurons.

\begin{proposition} \label{thm:dnf-special}
For any $n$, there exists a 1-layer ReLU net with $6$ neurons that computes $(n, 3)$-parity.
\end{proposition}

\begin{proof}
Let $x_1, x_2, x_3$ be the $3$ parity bits. We construct $\mathbf h \in \mathbb R^6$ as follows, where $\sigma$ is ReLU:
\begin{align*}
    h_1 &= \sigma(-x_1 + -x_2 + 10 x_3 - 9) \\
    h_2 &= \sigma(-x_1 - x_2 + x_3 - 2) \\
    h_3 &= \sigma(x_1 + x_2 + x_3 - 2) \\
    h_4 &= \sigma(x_1 - x_2 - 10 x_3 - 9) \\
    h_5 &= \sigma(-x_1 + x_2 - x_3 - 2) \\
    h_6 &= \sigma(x_1 - x_2 - x_3 - 2) .
\end{align*}
In the final layer, we assign $h_1$ and $h_4$ a weight of $-1$, and $h_2, h_3, h_5$, and $h_6$ a weight of $+10$.

To show correctness, we first characterize the logical condition that each neuron encodes:
\begin{align*}
    h_1 > 0 &\iff (x_1 = -1 \vee x_2 = -1) \wedge x_3 = 1 \\
    h_2 > 0 &\iff x_1 = -1 \wedge x_2 = -1 \wedge x_3 = 1 \\
    h_3 > 0 &\iff x_1 = 1 \wedge x_2 = 1 \wedge x_3 = 1 \\
    h_4 > 0 &\iff (x_1 = 1 \vee x_2 = -1) \wedge x_3 = -1 \\
    h_5 > 0 &\iff x_1 = -1 \wedge x_2 = 1 \wedge x_3 = -1 \\
    h_6 > 0 &\iff x_1 = 1 \wedge x_2 = -1 \wedge x_3 = -1 .
\end{align*}
In the final layer, $h_1$ and $h_4$ contribute a weight of $-1$ whenever the parity is negative (and in two other cases).
But in the four cases when the true parity is positive, one of the other neurons contributes a positive weight of $+10$.
Thus, the sign of the network output is correct in all $8$ cases. We conclude that this $6$-neuron network correctly computes the parity of $x_1, x_2$, and $x_3$.
\end{proof}

However, there is a $4$-neuron construction computing parity,\footnote{We thank anonymous reviewer 3kdq for demonstrating this construction.} which, interestingly, our networks do not find:

\begin{proposition} \label{thm:threshold}
    For any $n$, there exists a $1$-layer ReLU net with $4$ neurons that computes $(n, 3)$-parity.
\end{proposition}

\begin{proof}
    Let $X = x_1 + x_2 + x_3$ be the sum of the parity bits. We construct $\mathbf h \in \mathbb R^4$ as follows:
    \begin{align*}
        h_1 &= 1 \\
        h_2 &= \sigma(X - 1) \\
        h_3 &= \sigma(X + 1) \\
        h_4 &= \sigma(-X - 1) .
    \end{align*}
    In the final layer, we assign $h_1$ a weight of $1$, $h_2$ a weight of $2$, and $h_3$ and $h_4$ a weight of $-1$. We proceed by cases over the possible values of $X \in \{\pm 1, \pm 3\}$, which uniquely determines the parity:
    \begin{enumerate}
        \item \underline{$X = -3$:} Then there are three input bits with value $-1$, so the parity is $-1$. We see that $h_1 = 1$, $h_2 = 0$, $h_3 = 0$, and $h_4 = 2$. So the output is $h_1 - h_4 = -1$.
        \item \underline{$X = -1$:} Then there are two input bits with value $-1$, so the parity is $1$. We see that $h_1 = 1$, $h_2 = 0$, $h_3 = 0$, and $h_4 = 0$. So the output is $h_1 = 1$.
        \item \underline{$X = 1$:} Then there is one input with value $-1$, so the parity is $-1$. We see that $h_1 = 1$, $h_2 = 0$, $h_3 = 2$, and $h_4 = 0$. So the output is $h_1 - h_3 = -1$.
        \item \underline{$X = 3$:} Then there are no inputs with value $-1$, so the parity is $1$. We see that $h_1 = 1$, $h_2 = 2$, $h_3 = 4$, and $h_3 = 0$. So the output is $h_1 + 2h_2 - h_3 = 1$.
    \end{enumerate}
    We conclude that this $4$-neuron network correctly computes the parity of $x_1$, $x_2$, and $x_3$.
\end{proof}

\end{document}